\documentclass[conference]{IEEEtran}
\IEEEoverridecommandlockouts
\usepackage{booktabs} 

\usepackage[a4paper]{geometry}
\usepackage{blkarray}
\usepackage{multirow}

\usepackage{cite}
\usepackage{amsmath,amssymb,amsthm,amsfonts}
\usepackage{mathtools}
\usepackage{caption,subcaption}
\usepackage{graphicx}
\graphicspath{{./}}
\usepackage{textcomp}
\usepackage{xcolor}
\usepackage{bm}

\def\BibTeX{{\rm B\kern-.05em{\sc i\kern-.025em b}\kern-.08em
		T\kern-.1667em\lower.7ex\hbox{E}\kern-.125emX}}
\newtheorem{thm}{Theorem}
\newtheorem{lem}[thm]{Lemma}
\newtheorem{prop}[thm]{Proposition}

\DeclarePairedDelimiterX{\norm}[1]{\lVert}{\rVert}{#1}

\usepackage{algorithm}
\usepackage{algpseudocode}

\newtheoremstyle{problemstyle}
{\topsep} 
{\topsep}
{\itshape} 
{}         
{\bfseries}
{.}        
{.5em}     
{}         

\theoremstyle{problemstyle}

\AtBeginDocument{
	\setlength{\abovedisplayskip}{3pt}
	\setlength{\belowdisplayskip}{3pt}
	\setlength{\abovedisplayshortskip}{2pt}
	\setlength{\belowdisplayshortskip}{2pt}
}
\begin{document}

\title{A Theoretical Framework for Energy-Aware Gradient Pruning in Federated Learning}

\author{\IEEEauthorblockN{Emmanouil M.~Athanasakos}
\IEEEauthorblockA{\textit{National and Kapodistrian University of Athens} \\
\textit{Department of Informatics and Telecommunications}\\
Athens, Greece \\
emathan@di.uoa.gr}

}

\maketitle

\begin{abstract}
Federated Learning (FL) is constrained by the communication and energy limitations of decentralized edge devices. While gradient sparsification via Top-K magnitude pruning effectively reduces the communication payload, it remains inherently energy-agnostic. It assumes all parameter updates incur identical downstream transmission and memory-update costs, ignoring hardware realities. We formalize the pruning process as an energy-constrained projection problem that accounts for the hardware-level disparities between memory-intensive and compute-efficient operations during the post-backpropagation phase. We propose Cost-Weighted Magnitude Pruning (CWMP), a selection rule that prioritizes parameter updates based on their magnitude relative to their physical cost. We demonstrate that CWMP is the optimal greedy solution to this constrained projection and provide a probabilistic analysis of its global energy efficiency. Numerical results on a non-IID CIFAR-10 benchmark show that CWMP consistently establishes a superior performance-energy Pareto frontier compared to the Top-K baseline.
\end{abstract}


\section{Introduction}
\label{sec:intro}
\vspace{-0.2em}
Federated Learning (FL) has emerged as the definitive paradigm for privacy-preserving decentralized optimization, enabling collaborative model training without the exfiltration of sensitive local data \cite{mcmahan2017communication}. As neural architectures scale in complexity, the research focus has shifted toward sustainability~\cite{schwartz2020green}, where resource efficiency is prioritized alongside model performance. In wireless edge environments, FL is bottlenecked by the dual constraints of limited communication bandwidth and finite battery life \cite{gunduz2022beyond}.

To mitigate communication overhead, gradient sparsification has become the de facto standard \cite{aji2017sparse, stich2018sparsified}. Techniques such as Deep Gradient Compression \cite{lin2018deep} and single-shot pruning \cite{lee2018snip} rely primarily on magnitude-based selection (Top-K) to reduce the update entropy. While effective for bandwidth reduction, these methods are inherently energy-agnostic: they assume a uniform counting measure for the cost of model updates, implicitly treating every parameter as having an identical thermodynamic footprint.

A recent comprehensive survey of energy efficiency in FL \cite{sezgin2025survey} reveals a rich landscape of high-level orchestration strategies, including Lyapunov-based resource allocation \cite{yang2021energy}, broadband analog aggregation \cite{zhu2020broadband}, and adaptive client selection \cite{wang2019adaptive}. Recently in \cite{pereira2025energy} the authors reduced the total energy footprint of AIoT networks through clustering-informed client selection. While these works address system-level heterogeneity \cite{li2020federated, karimireddy2020scaffold}, the underlying sparsification mechanism remain primitive. Recent results in the information geometry of local generalization dynamics \cite{athanasakos_aistats} and natural gradient descent \cite{amari1998natural} suggest that the true utility of a parameter update is governed by the local curvature of the loss landscape. In resource-constrained systems, this utility must be weighed against the physical cost of acquisition. Drawing inspiration from the Information Bottleneck principle\cite{ChechikTishby_NIPS2002} and the Lottery Ticket Hypothesis \cite{frankle2018lottery}, we argue that an optimal update strategy must formally co-optimize for informational relevance and physical energy.

In this work, we address the energy-agnosticism of standard pruning by formalizing the update process as an energy-constrained projection problem. We introduce a computational energy measure to model the hardware-level cost disparity between memory-intensive and compute-efficient parameters during the post-backpropagation phase. The core contributions are:
\begin{itemize}
	\item  We define a framework for energy-aware learning, characterizing parameter transmission and memory-update costs as atoms of a discrete energy measure on the parameter index set.
	\item  We derive CWMP, a selection rule that ranks parameters by their \textit{efficiency density}. We prove that CWMP is the optimal greedy solution to the resource-constrained projection.
	\item Through numerical experiments on a non-IID CIFAR-10 benchmark, we demonstrate that CWMP establishes a superior Performance-Energy Pareto frontier and follows a convergence trajectory nearly identical to the dense-informed Top-K.
\end{itemize}
\section{Problem Formulation}
\label{sec:background}

We begin by formalizing the FL process and the standard method of gradient sparsification. We then introduce a model for computational cost, revealing a critical oversight in current energy-agnostic pruning techniques.

\noindent Let $(\Omega, \mathcal{F}, \mathbb{P})$ be a probability space. We consider the collaborative optimization of a global parameter $\mathbf{w} \in \mathbb{R}^d$. The objective is to minimize a global risk functional $F(\mathbf{w}) \triangleq \sum_{k=1}^{K} p_k F_k(\mathbf{w})$, where each local loss $F_k(\mathbf{w}) \triangleq \int \ell(\mathbf{w}; \xi) d\mathcal{D}_k(\xi)$ is defined with respect to a local data measure $\mathcal{D}_k$ on a client $k$. The standard approach for this task is Federated Averaging (FedAvg) \cite{mcmahan2017communication}. Let $\{\mathcal{F}_t\}_{t \ge 0}$ be the filtration representing the history of the learning process up to round $t$. The protocol proceeds as follows:
\begin{enumerate}
	\item The server broadcasts the current global parameters $\mathbf{w}^{(t)}$ to the clients.
	\item Each client $k$ computes a local update via Stochastic Gradient Descent (SGD) on its local measure $\mathcal{D}_k$. This yields a local gradient estimate $\mathbf{g}_k^{(t)}$, which is an $\mathcal{F}_t$-measurable random variable.
	\item Clients transmit a processed version of their gradients, $\tilde{\mathbf{g}}_k^{(t)}$, back to the server.
	\item The server updates the global state: $\mathbf{w}^{(t+1)} = \mathbf{w}^{(t)} - \eta \sum_{k=1}^{K} p_k \tilde{\mathbf{g}}_k^{(t)}$.
\end{enumerate}
\noindent This formulation acknowledges that the local updates are stochastic realizations of the underlying local measures, where the non-IID nature of the data is captured by the divergence between the measures $\{\mathcal{D}_k\}$. To mitigate communication bottlenecks in Step 3, clients apply a sparsification operator. We formalize this as a projection $\Pi_S : \mathbb{R}^d \to \mathbb{R}^d$ onto the coordinate subspace defined by a support set $S \subset \{1, \dots, d\}$ of cardinality $k \ll d$. The sparse update is given by:
\begin{equation}
	\tilde{\mathbf{g}}_k^{(t)} = \Pi_S (\mathbf{g}_k^{(t)}) = \mathbf{m}_k^{(t)} \odot \mathbf{g}_k^{(t)},
\end{equation}
where $\mathbf{m}_k^{(t)} \in \{0, 1\}^d$ is a binary pruning mask with $||\mathbf{m}_k^{(t)}||_0 = k$. The standard for generating this mask is Top-K magnitude pruning. Let $\mathcal{I}_k^{(t)}$ be the set of indices corresponding to the $k$ elements in $\mathbf{g}_k^{(t)}$ with the largest absolute values. The Top-K mask is defined as $m_{k,j}^{(t)} = 1$ if $j \in \mathcal{I}_k^{(t)}$, and 0 otherwise. This rule is effective for communication efficiency but is energy-agnostic: it assumes that every parameter $w_j$ incurs an identical physical cost to compute and transmit. The subsequent post-backpropagation phases (payload index-encoding, wireless transmission, and downstream memory-update accesses at the aggregator) exhibit severe hardware-level cost disparities that Top-K ignores. To address this, we introduce a computational energy measure $\mu$ on the power set of parameter indices to model these bottlenecks. For any $S$, $\mu(S)$ quantifies the physical energy consumed to process the selected components. This measure is characterized by its density with respect to the counting measure, represented by a parameter cost vector $\mathbf{c} \in \mathbb{R}_+^d$, such that $c_j$ represents the physical cost of the $j$-th parameter. The total energy footprint of an update $\tilde{\mathbf{g}}_k^{(t)}$ is modeled as:
\begin{equation}
	\mathcal{E}(\tilde{\mathbf{g}}_k^{(t)}) \triangleq \mu(S) = \sum_{j=1}^{d} m_{k,j}^{(t)} c_j = {\mathbf{m}_k^{(t)}}^\top \mathbf{c}.
	\label{eq:cost_measure}
\end{equation}
\noindent Under Top-K pruning, the communication cost $||\mathbf{m}_k^{(t)}||_0$ is fixed, but the resulting energy $\mathcal{E}(\tilde{\mathbf{g}}_k^{(t)})$ is an uncontrolled, stochastic quantity. This leads to suboptimal resource allocation, where a client expends its finite energy budget on gradient components with high magnitude but poor informational utility relative to their physical processing cost. This motivates the following optimization problem:
\begin{align}
	S^* = \operatorname*{arg\,max}_{S \subseteq \{1,\dots,d\}} \quad & \left\| \Pi_S (\mathbf{g}_k^{(t)}) \right\|_1 \label{eq:constrained_opt} \\
	\text{subject to} \quad & |S| \le k_{\text{sparsity}}, \tag{C1} \\
	& \mu(S) \le E_{\text{budget}}. \tag{C2}
\end{align}
Solving this requires a selection rule that considers the density of the gradient power with respect to the energy measure $\mu$, leading to the algorithm proposed in the following section.

\section{Cost-Weighted Magnitude Pruning}
\label{sec:cwmp}

The formulation in Section \ref{sec:background} implies that an optimal pruning strategy must maximize the utility of the update signal relative to the physical resources consumed. A gradient component, however large in magnitude, may represent a suboptimal use of the communication budget if its underlying computational cost is disproportionately high. This motivates a selection process governed by update efficiency, which is the ratio of its local contribution to the learning dynamics versus its associated energy cost. While the exact contribution of a parameter to generalization is determined by the local information geometry \cite{athanasakos_aistats}, the gradient magnitude $|g_{k,j}^{(t)}|$ serves as a robust and tractable first-order proxy for the update's informational utility. To this end, we introduce the \textit{efficiency score}, defined as the density of the gradient magnitude with respect to the cost measure $\mu$:
\begin{equation}
	s_{k,j}^{(t)} \triangleq \frac{|g_{k,j}^{(t)}|}{c_j},
	\label{eq:score}
\end{equation}
where $c_j > 0$ is the parameter cost atom. This score normalizes the importance of each gradient component by its physical expense. In this regime, a component must exhibit a significantly higher magnitude to be selected if it resides in a high-cost region of the parameter space (e.g., within a massive, memory-intensive fully-connected layer). Using this score, we propose CWMP. The selection process remains analogous to standard sparsification, but the ranking criterion is replaced by \eqref{eq:score}. Let $\mathcal{J}_k^{(t)}$ be the set of indices corresponding to the $k$ largest efficiency scores $\{s_{k,j}^{(t)}\}_{j=1}^d$. The CWMP pruning mask is formally defined as:
\begin{equation}
	m_{k,j}^{(t)} = 
	\begin{cases} 
		1 & \text{if } j \in \mathcal{J}_k^{(t)} \\
		0 & \text{otherwise}
	\end{cases}
	\label{eq:cwmp}
\end{equation}
The client-side procedure is detailed in Alg. \ref{alg:cwmp}.
\begin{algorithm}[h!]
	\caption{CWMP}
	\label{alg:cwmp}
	\begin{algorithmic}[1]
		\State \textbf{Require:} Global state $\mathbf{w}^{(t)}$, local measure $\mathcal{D}_k$, budget $k$, Cost Vector $\mathbf{c}$.
		\State \textbf{Client training:}
		\State $\mathbf{g}_k \leftarrow \text{LocalSGD}(\mathbf{w}^{(t)}, \mathcal{D}_k)$ \Comment{Estimate $\mathcal{F}_t$-measurable gradient}
		\State $\mathbf{s}_k \leftarrow |\mathbf{g}_k| \oslash \mathbf{c}$ \Comment{Compute efficiency scores (element-wise)}
		\State $\mathcal{J}_k \leftarrow \text{TopKIndices}(\mathbf{s}_k, k)$ \Comment{Rank by efficiency density}
		\State $\mathbf{m}_k \leftarrow \text{ConstructMask}(\mathcal{J}_k)$
		\State $\tilde{\mathbf{g}}_k \leftarrow \mathbf{m}_k \odot \mathbf{g}_k$ \Comment{Generate energy-aware projection}
		\State \textbf{return} $\tilde{\mathbf{g}}_k$ to server
	\end{algorithmic}
\end{algorithm}
CWMP is a generalization of Top-K pruning; for a uniform cost measure ($\mathbf{c} = \alpha \mathbf{1}$), the score $s_{k,j}$ becomes proportional to $|g_{k,j}|$, causing the CWMP selection to converge to the standard Top-K mask. For non-uniform costs, CWMP reallocates the communication budget towards parameters with higher arithmetic intensity and lower memory-access energy.

\subsection{Theoretical Justification}

Problem in \eqref{eq:constrained_opt} is a multidimensional 0/1 Knapsack problem, which is strongly NP-hard. To derive a computationally tractable $O(d \log k)$ selection rule, we isolate the energy-constrained subproblem. The following lemma establishes the optimal scoring metric for maximizing gradient mass under energy budget.
\begin{lem}
	Let $\mathbf{g} \in \mathbb{R}^d$ be a gradient vector and $\mathbf{c} \in \mathbb{R}_+^d$ be a strictly positive cost vector. Consider the energy-constrained gradient mass maximization over a support set $S \subseteq \{1, \dots, d\}$:
	\begin{equation}
		\max_{S} \sum_{j \in S} |g_j| \quad \text{subject to} \quad \sum_{j \in S} c_j \le E_{\text{budget}},
		\label{eq:knapsack}
	\end{equation}
	where $E_{\text{budget}} > 0$. The optimal selection policy for the continuous relaxation of \eqref{eq:knapsack} is to rank and select parameters in descending order of the efficiency score $s_j = |g_j|/c_j$.
\end{lem}
\begin{IEEEproof}
	The proof is presented in Appendix A.
\end{IEEEproof}
Lemma 1 establishes $s_j$ as the theoretically optimal density for the energy-constrained relaxation. To satisfy the joint constraints of \eqref{eq:constrained_opt}, CWMP employs this optimal density as the ranking criterion while applying a hard truncation at the $k$-th index. This ensures compliance with (C1) while providing a high-fidelity greedy approximation to the multidimensional optimum. The following proposition characterizes the expected global energy profile of this selection rule compared to the energy-agnostic baseline.
\begin{prop}
	Let the gradient components $\{|g_j|\}_{j=1}^d$ and parameter costs $\{c_j\}_{j=1}^d$ be drawn from independent distributions. For a fixed communication budget $k$, let $\mathbf{m}_{\text{TopK}}$ and $\mathbf{m}_{\text{CWMP}}$ be the pruning masks generated by the Top-K and CWMP algorithms, respectively. The expected computational energy cost satisfies:
	\begin{equation}
		\mathbb{E}\left[ \mathbf{m}_{\text{CWMP}}^\top \mathbf{c} \right] \le \mathbb{E}\left[ \mathbf{m}_{\text{TopK}}^\top \mathbf{c} \right],
	\end{equation}
	where equality holds if and only if the cost measure $\mu$ is the counting measure.
\end{prop}
\begin{IEEEproof}
 The proof is presented in Appendix B.
\end{IEEEproof}
\noindent This result provides the theoretical justification for the empirical results that follow. By construction, CWMP is biased toward selecting indices with high information density per unit of energy, ensuring that the device's energy budget is allocated only to the most efficient gradient updates.

\section{Experimental Validation}
\label{sec:experiments}

We evaluate our method on the CIFAR-10 image classification task and utilize a ResNet-18 architecture \cite{he2016deep} adapted for $32 \times 32$ inputs.

\paragraph{Federated Configuration}
We simulate a federation of $K=10$ clients with statistical heterogeneity modeled via a Dirichlet distribution with concentration $\alpha=0.5$. The global model is trained for $T=50$ rounds. Each client performs local optimization using SGD with a learning rate of 0.05, momentum of 0.9, and a batch size of 64.

\paragraph{Memory-Centric Cost Model}
Vector $\mathbf{c}$ is defined based on the physical energy disparity between memory access and computation in edge hardware \cite{horowitz20141}. Generating the dense gradient requires a full backward pass; therefore, our cost model explicitly targets the post-backpropagation bottlenecks. In modern CMOS accelerators, the energy cost of a DRAM fetch—required for the massive, non-reused weights in the classifier—dominates the budget\cite{han2015deep}. Conversely, convolutional filters exhibit high arithmetic intensity and SRAM reuse\cite{sze2017efficient}. We thus define:
\begin{equation}
	c_j = 
	\begin{cases} 
		5.0 & \text{if } w_j \in \text{Classifier Layers} \\
		1.0 & \text{if } w_j \in \text{Feature Extr. (Conv) Layers}
	\end{cases}
\end{equation}
Measure $\mu$ decouples the informational utility of a parameter from its physical processing cost, allowing us to evaluate the impact of constraint (C2).

\paragraph{Baselines}
We compare CWMP against Top-K Magnitude Pruning, which serves as the energy-agnostic baseline. Both methods are evaluated across a range of sparsity levels $k \in \{1\%, 5\%, 10\%, 20\%\}$ to trace the Pareto frontier.

\begin{figure}[htb!]
	\centering
	\includegraphics[width=\columnwidth]{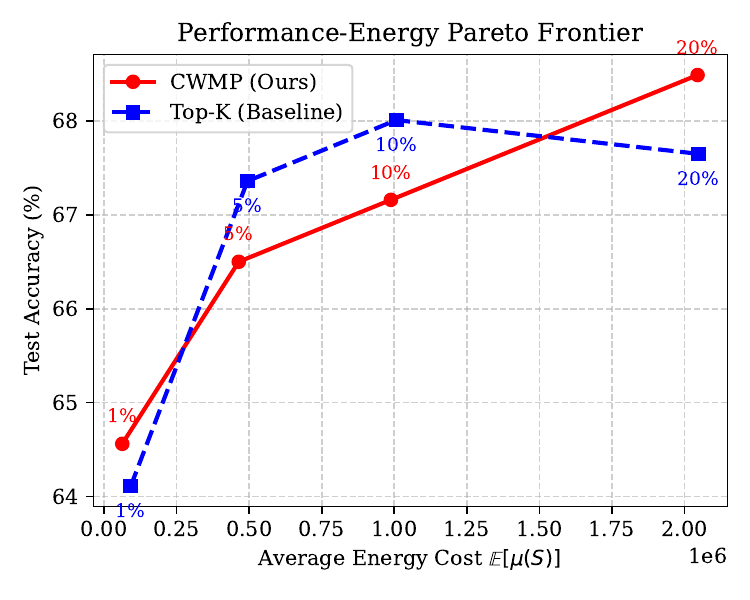}
	\caption{\small Performance-Energy Pareto Frontier. CWMP is the dominant strategy in the extreme scarcity regime (1\%) and maintains a superior accuracy-per-energy profile as the budget increases, avoiding the performance regression observed in Top-K.}
	\label{fig:pareto}
\end{figure}

\noindent As illustrated in Fig. \ref{fig:pareto}, at a 1\% communication budget, CWMP acts as a strictly dominant strategy. It achieves a test accuracy of 64.56\% while incurring a normalized energy cost of $0.06 \times 10^6$ units. In contrast, the energy-agnostic Top-K baseline requires $0.09 \times 10^6$ energy units—a 48\% increase—to achieve a lower accuracy of 64.11\%. This confirms that when resources are most scarce, energy-blind pruning wastes the budget on high-cost updates that offer marginal generalization returns. At the 20\% sparsity level, we observe a divergence in the behavior of the two strategies. While the Top-K baseline incurs a significant increase in computational cost, its test accuracy exhibits a slight regression (from 68.01\% to 67.65\%). In this specific architectural setup, this suggests that magnitude-based selection may prioritize high-variance updates in the massive classifier layers, leading to marginal overfitting. 
Conversely, the CWMP trajectory remains robust, reaching a peak accuracy of 68.49\%. This suggests that the cost-weighted selection rule acts as an implicit structural regularizer. By imposing a higher "price" on the memory-intensive parameters of the fully-connected layers, CWMP naturally shifts the learning focus toward the more compact and robust features in the convolutional layers. While these results are observed within the context of ResNet-18, they highlight the potential for energy-awareness to serve as a functional prior in resource-constrained learning.

Fig. \ref{fig:acc_vs_rounds} confirms that incorporating the cost measure does not destabilize the stochastic dynamics. CWMP follows a convergence trajectory nearly identical to the dense-informed Top-K, reaching the 70\% accuracy threshold within 15 rounds,but with a significantly lower cumulative energy footprint. 
\begin{figure}[htb!]
	\centering
	\includegraphics[width=\columnwidth]{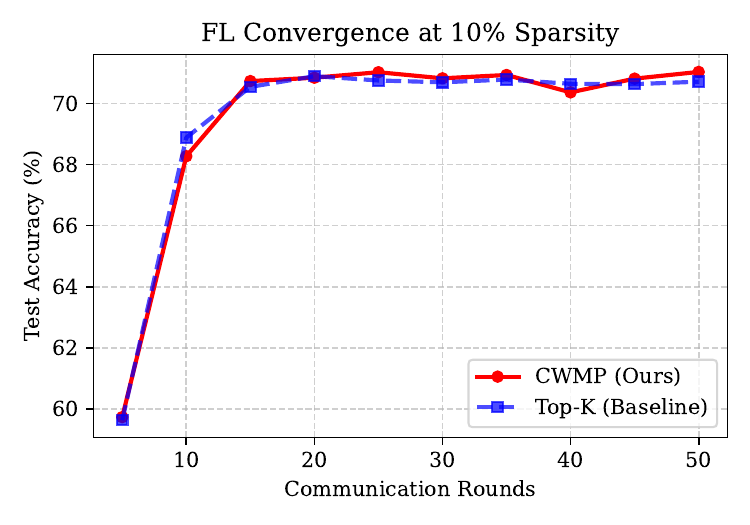}
	\caption{\small Convergence Dynamics at 10\% Sparsity. CWMP matches the convergence rate of the Top-K baseline while achieving a higher accuracy, demonstrating a more efficient allocation of the fixed communication budget.}
	\label{fig:acc_vs_rounds}
\end{figure}

\section{Conclusion and Future Work}
\label{sec:conclusion}
In this work, we addressed the energy-agnosticism of standard gradient sparsification in FL. By formalizing the update process as an energy-constrained projection, we proposed a selection rule that prioritizes parameters based on their efficiency density rather than raw magnitude. Empirical evaluations on a non-IID CIFAR-10 benchmark demonstrated that CWMP establishes a strictly superior performance-energy Pareto frontier, acting as an implicit regularizer that prevents the model from over-allocating resources to high-cost, noisy updates. Our evaluation isolates this sparsification primitive to highlight the pure impact of energy-awareness. Because CWMP is entirely orthogonal to error compensation mechanisms, it serves as a replacement for the standard Top-K operator. It can be readily integrated into  Error-Feedback frameworks \cite{lin2018deep} to further enhance the Pareto efficiency of modern federated optimizers. 

Ultimately, the success of this cost-aware projection demonstrates that physical energy metrics are not merely peripheral constraints, but functional priors that guide the model toward more robust representations. This motivates a deeper, information-theoretic investigation to derive provably optimal compression strategies that formally co-optimize learning performance, communication bandwidth, and hardware-level energy expenditure.

\bibliographystyle{IEEEtran} 
\bibliography{references_manos_Fed_IB_spawc} 


\appendices

\section{Proof of Lemma 1}
\begin{IEEEproof}
We represent the selection of the support set $S$ via a binary mask $\mathbf{m} \in \{0, 1\}^d$, where $m_j = 1$ if $j \in S$ and $0$ otherwise. To identify the optimal selection rule, we consider the continuous relaxation where $m_j \in [0, 1]$. The resulting optimization problem is a linear program (LP):
\begin{align}
	\max_{\mathbf{m} \in \mathbb{R}^d} \quad & \sum_{j=1}^{d} m_j |g_j| \label{eq:lp_obj} \\
	\text{subject to} \quad & \sum_{j=1}^{d} m_j c_j \le E_{\text{budget}}, \label{eq:lp_c1} \\
	& 0 \le m_j \le 1, \quad \forall j \in \{1, \dots, d\}. \label{eq:lp_c2}
\end{align}
	
\noindent Let $\lambda \ge 0$ be the Lagrange multiplier associated with the energy constraint \eqref{eq:lp_c1}. Let $\alpha_j \ge 0$ and $\beta_j \ge 0$ be the multipliers for the lower bound ($-m_j \le 0$) and upper bound ($m_j - 1 \le 0$) constraints of \eqref{eq:lp_c2}, respectively. The Lagrangian is given by
\begin{align}
	\mathcal{L}(\mathbf{m}, &\lambda, \boldsymbol{\alpha}, \boldsymbol{\beta}) \nonumber \\
	&= \sum_{j=1}^d m_j |g_j| - \lambda \left( \sum_{j=1}^d m_j c_j - E_{\text{budget}} \right)\nonumber \\  
	& + \sum_{j=1}^d \alpha_j m_j - \sum_{j=1}^d \beta_j (m_j - 1). \label{eq:langran}
\end{align}
The stationarity condition for each component $m_j$ is:
\begin{align}
	\frac{\partial \mathcal{L}}{\partial m_j}& = |g_j| - \lambda c_j + \alpha_j - \beta_j = 0 \nonumber \\
	&\implies |g_j| - \lambda c_j = \beta_j - \alpha_j.
	\label{eq:stationarity}
\end{align}
From the complementary slackness, the following conditions hold for all $j$:
\begin{align}
	\alpha_j m_j^* &= 0, \label{eq:cs1} \\
	\beta_j (m_j^* - 1) &= 0. \label{eq:cs2}
\end{align}

\noindent We now analyze the possible states of the optimal allocation $m_j^*$ based on the stationarity condition \eqref{eq:stationarity} and complementary slackness:

\begin{enumerate}
	\item Case $m_j^* = 1$ (Fully selected): From \eqref{eq:cs1}, we must have $\alpha_j = 0$. Since $\beta_j \ge 0$, equation \eqref{eq:stationarity} implies $|g_j| - \lambda c_j = \beta_j \ge 0$. Dividing by the strictly positive cost $c_j$, we obtain:
	\begin{equation}
		\frac{|g_j|}{c_j} \ge \lambda.
	\end{equation}
	
	\item Case $m_j^* = 0$ (Not selected): From \eqref{eq:cs2}, we must have $\beta_j = 0$. Since $\alpha_j \ge 0$, equation \eqref{eq:stationarity} implies $|g_j| - \lambda c_j = -\alpha_j \le 0$. Thus:
	\begin{equation}
		\frac{|g_j|}{c_j} \le \lambda.
	\end{equation}
	
	\item Case $0 < m_j^* < 1$ (Fractionally selected): From \eqref{eq:cs1} and \eqref{eq:cs2}, we must have $\alpha_j = 0$ and $\beta_j = 0$. Equation \eqref{eq:stationarity} therefore requires:
	\begin{equation}
		\frac{|g_j|}{c_j} = \lambda.
	\end{equation}
\end{enumerate}
These conditions dictate that the optimal continuous allocation $m_j^*$ is entirely determined by comparing the ratio $|g_j|/c_j$ against the optimal dual variable $\lambda$. To satisfy the constraints while maximizing the objective, the optimal policy must assign $m_j^* = 1$ to parameters with the highest $|g_j|/c_j$ ratios, iteratively lowering the threshold $\lambda$ until the energy budget is exhausted. Therefore, selecting parameters in descending order of the efficiency score $s_j = |g_j|/c_j$ precisely matches the optimal greedy policy for the relaxed energy-constrained projection problem and that completes the proof.
\end{IEEEproof}

\section{Proof of Proposition 2}
\begin{IEEEproof}
	We model the gradient magnitudes $G_j \triangleq |g_j|$ and the parameter costs $C_j \triangleq c_j$ as sequences of i.i.d. random variables drawn from continuous, strictly positive distributions $F_G$ and $F_C$, respectively. Furthermore, we assume\footnote{While architectural priors in deep networks often induce correlations between gradient magnitudes and layer-wise costs, this assumption establishes a baseline analytical bound by mathematically isolating the energy savings attributable strictly to the algorithmic selection rule.} $G_j$ and $C_j$ are independent of each other for all $j \in \{1, \dots, d\}$. Let $k$ be the fixed communication budget (sparsity level). The selection policies for Top-K and CWMP define binary masks $\mathbf{m}^{\text{TopK}}, \mathbf{m}^{\text{CWMP}} \in \{0, 1\}^d$ subject to the cardinality constraint $\sum_{j=1}^d m_j = k$. The Top-K algorithm selects indices based solely on the gradient magnitude $G_j$. The mask is defined via the indicator function:
	\begin{equation}
		m_j^{\text{TopK}} = \mathbb{I}\left\{ G_j \ge \tau_G \right\},
	\end{equation}
	where $\tau_G$ is the $(d-k+1)$-th order statistic of the sequence $\{G_1, \dots, G_d\}$.

Because the selection threshold $\tau_G$ and the sequence $G_j$ are entirely independent of the hardware costs $C_j$, the random variables $m_j^{\text{TopK}}$ and $C_j$ are independent. Therefore, the expected energy cost for a single parameter is separable:
\begin{equation}
	\mathbb{E}\left[ m_j^{\text{TopK}} C_j \right] = \mathbb{E}\left[ m_j^{\text{TopK}} \right] \mathbb{E}\left[ C_j \right].
\end{equation}
By the exchangeability of the i.i.d. variables, the joint distribution is invariant under any permutation of the indices \cite{kingman1978mathematical}. Consequently, the marginal probability $\mathbb{P}(j\in S)$ must be identical for all $j\in \{1,\ldots,d\}$. Given the cardinality constraint $\sum_{j=1}^{d} \mathbb{E}[m_j^{\text{TopK}}] = k$, it follows that $\mathbb{E}[m_j^{\text{TopK}}] = \frac{k}{d}$. Summing over all $d$ yields the baseline expected energy:
\begin{equation}
	\mathbb{E}\left[ (\mathbf{m}^{\text{TopK}})^\top \mathbf{c} \right] = \sum_{j=1}^d \mathbb{E}\left[ m_j^{\text{TopK}} C_j \right] = k \mathbb{E}[C_j].
	\label{eq:expected_topk}
\end{equation}

\noindent The CWMP algorithm selects indices based on the efficiency score $S_j \triangleq G_j / C_j$. The mask is defined as:
\begin{equation}
	m_j^{\text{CWMP}} = \mathbb{I}\left\{ S_j \ge \tau_S \right\},
\end{equation}
where $\tau_S$ is the $(d-k+1)$-th order statistic of the sequence $\{S_1, \dots, S_d\}$. Unlike Top-K, the selection variable $m_j^{\text{CWMP}}$ is strongly coupled to the cost $C_j$. To isolate this dependency, we define the conditional selection probability given a fixed cost realization $C_j = c$:
\begin{align}
	\phi(c) &\triangleq \mathbb{E}\left[ m_j^{\text{CWMP}} \mid C_j = c \right] \nonumber \\
	&= \mathbb{P}\left( \frac{G_j}{c} \ge \tau_S \mid C_j = c \right).
\end{align}

Because the ratio $G_j/c$ is strictly decreasing with respect to $c$, for $G_j > 0$, a higher cost $c$ lowers the efficiency score $S_j$, making it less likely to exceed the top-$k$ threshold $\tau_S$. Consequently, $\phi(c)$ is monotonically decreasing with respect to $c$. The expected energy of the $j$-th parameter under CWMP is:
\begin{align}
	\mathbb{E}\left[ m_j^{\text{CWMP}} C_j \right] &= \mathbb{E}_{C_j}\left[ \mathbb{E}\left[ m_j^{\text{CWMP}} C_j \mid C_j \right] \right] \nonumber \\ 
	&= \mathbb{E}_{C_j}\left[ C_j \phi(C_j) \right]. \label{eq:chebyshev_expect}
\end{align}
We evaluate the expectation \eqref{eq:chebyshev_expect} utilizing Chebyshev's association inequality\cite{boucheron2013concentration}.

\noindent Consider the random variable $C_j$. The identity function $f(c) = c$ is strictly monotonically increasing, while the conditional probability function $\phi(c)$ is monotonically decreasing. The covariance of two functions with opposite monotonicities applied to the same random variable is strictly non-positive:
\begin{equation}
	\text{Cov}\left( C_j, \phi(C_j) \right) \le 0.
\end{equation}
Expanding the covariance formulation yields:
\begin{align}
	\mathbb{E}\left[ C_j \phi(C_j) \right] &- \mathbb{E}[C_j]\mathbb{E}[\phi(C_j)] \le 0 \nonumber \\ 
	&\implies \mathbb{E}\left[ C_j \phi(C_j) \right] \le \mathbb{E}[C_j]\mathbb{E}[\phi(C_j)].
\end{align}
By the law of total expectation, $\mathbb{E}[\phi(C_j)] = \mathbb{E}[m_j^{\text{CWMP}}]$ and by the symmetry of the selection $\sum m_j^{\text{CWMP}} = k$, we maintain $\mathbb{E}[m_j^{\text{CWMP}}] = \frac{k}{d}$. Therefore:
\begin{equation}
	\mathbb{E}\left[ m_j^{\text{CWMP}} C_j \right] \le \frac{k}{d} \mathbb{E}[C_j].
\end{equation}
Summing this upper bound across all $d$ gives:
\begin{equation}
	\mathbb{E}\left[ (\mathbf{m}^{\text{CWMP}})^\top \mathbf{c} \right] \le k \mathbb{E}[C_j].
\end{equation}
Substituting the Top-K baseline derived in \eqref{eq:expected_topk} concludes the proof:
\begin{equation}
	\mathbb{E}\left[ (\mathbf{m}^{\text{CWMP}})^\top \mathbf{c} \right] \le \mathbb{E}\left[ (\mathbf{m}^{\text{TopK}})^\top \mathbf{c} \right].
\end{equation}
The equality holds if and only if $\text{Cov}( C_j, \phi(C_j) ) = 0$. Since $f(c) = c$ is strictly increasing, the covariance is zero if and only if the random variable $C_j$ is a constant almost surely. Thus equality is achieved when the parameter cost vector $\mathbf{c}$ is uniform, corresponding to the unweighted counting measure. 
\end{IEEEproof}
\end{document}